\documentclass[letterpaper, 10 pt, conference]{ieeeconf}  % Comment this line out if you need a4paper

\IEEEoverridecommandlockouts                              % This command is only needed if
\usepackage[colorinlistoftodos, textsize=tiny]{todonotes}
\usepackage{subcaption}
\usepackage{float}
\usepackage{amsmath}
\usepackage{amssymb}  % assumes amsmath package installed
\usepackage[vlined,linesnumbered]{algorithm2e}

\usepackage{balance}

\SetAlFnt{\footnotesize}
\SetAlCapFnt{\small}
\SetAlCapNameFnt{\small}
\SetArgSty{textnormal}

\SetCommentSty{mycommfont}

\newtheorem{theorem}{Theorem}
\newtheorem{lemma}[theorem]{Lemma}
\newtheorem{assumption}[theorem]{Assumption}

\newcommand{\Problem}[2]{{_{#1}\mathcal{P}_{#2}}}
\newcommand{\inProblem}[2]{{$\Problem{#1}{#2}$}}

\newcommand{\traj}[1]{\tau_{#1}}
\newcommand{\trajset}{\mathcal{S}}
\newcommand{\Cost}[1]{{\mathcal{C}\left(#1\right)}}
\newcommand{\state}[1]{{x_{#1}}}
\newcommand{\instate}[1]{$\state{#1}$}
\newcommand{\statespace}{\mathcal{X}}
\newcommand{\map}[1]{{\mathcal{M}_{#1}}}
\newcommand{\obs}[1]{\map{#1}}
\newcommand{\Bound}[1]{{\mathcal{B}(\obs{#1})}}
\newcommand{\DeltaBound}[1]{{\tilde{\mathcal{B}}(\obs{#1})}}

\newcommand{\Real}[1]{{\mathbb R}^{#1}}

\newcommand{\reals}{\mathbb{R}}

\newcommand{\ie}{{\em i.e.}}
\newcommand{\eg}{{\em e.g.}}
\newcommand{\etal}{{\em et. al.}}

\let\oldmarginpar\marginpar
\renewcommand\marginpar[1]{\oldmarginpar[\raggedleft\footnotesize #1]%
{\raggedright\footnotesize #1}}

\DeclareMathOperator*{\argmin}{arg\,min}

\title{\LARGE \bf
Perception-driven sparse graphs for optimal motion planning}

\author{Thomas Sayre-McCord$^{1}$ \and Sertac Karaman$^{1}$% <-this % stops a space
  \thanks{$^{1}$Laboratory for Information and Decision Systems (LIDS), Massachusetts Institute of Technology (MIT), Cambridge, MA 02139, USA.}%
}

\begin{document}

\maketitle
\thispagestyle{empty}
\pagestyle{empty}

\begin{abstract}
Most existing motion planning algorithms assume that a map (of some quality) is fully determined prior to generating a motion plan.
In many emerging applications of robotics, \eg, fast-moving agile aerial robots with constrained embedded computational platforms and visual sensors, dense maps of the world are not immediately available, and they are computationally expensive to construct.
We propose a new algorithm for generating plan graphs which couples the perception and motion planning processes for computational efficiency. In a nutshell, the proposed algorithm iteratively switches between the planning sub-problem and the mapping sub-problem, each updating based on the other until a valid trajectory is found. The resulting trajectory retains a provable property of providing an optimal trajectory with respect to the full (unmapped) environment, while utilizing only a fraction of the sensing data in computational experiments.
\end{abstract}

%!TEX root = main.tex
\section{Introduction} \label{section:introduction}

Both motion planning and mapping are fundamental problems of robotics. 
The problem of mapping is to create an accurate representation of obstacles around a robot based on sensory measurements, and the problem of motion planning is to find a dynamically-feasible trajectory around these obstacles.
These two problems are intimately linked.
A high-quality motion plan often requires working with an accurate high-resolution map, which requires extensive processing of large amounts of sensory data. 

Unfortunately, both the mapping problem and the motion planning problem can require significant computational resources. The computational constraints are even more pronounced for small vehicles attempting to traverse complex environments rapidly. Due to the small size of the vehicles, relatively limited computational platforms can be carried on board. Due to the fast operation of the vehicles, there is little time that can be devoted to computation.
In particular, typical mapping methods using camera data, \eg, stereo reconstruction, structure from motion, and learning based techniques, are all computationally burdensome, and their computation scales directly with the amount of area they need to map.

Due to the computation effort devoted to mapping there has been increasing interest in methods that attempt to minimize the processing of sensory data.
For instance, the pushbroom stereo method \cite{Barry2012safety} avoids full stereo depth computation for each stereo pair by integrating over time, and the NanoMap method \cite{Florence2017} maintains data in a sensor frame to avoid explicitly integrating it into a map.

In this paper, we consider a joint mapping-and-planning problem, in which the sensory data is processed for mapping only when it is necessary for planning, with the aim of minimizing the computational costs (for both mapping and planning), while maintaining the completeness and optimality guarantees of motion planning.
The algorithms that solve this problem are well suited for online settings that require small vehicles to rapidly traverse complex environments that are unknown {\em a priori}, but revealed in an online manner.

In many implementations for robot navigation, the coupling between motion planning and the obstacle map is through a search graph, \eg~\cite{kuwata2009real,hornung2012anytime}. The nodes of the graph typically consist of a set of robot states with the edges representing collision-free, dynamically-feasible trajectories between these states.
There is a vast literature on the {\em construction} of this discrete graph from the continuous robot dynamics and its environment. Common algorithms include regularized discretizations of the state space into state lattices, roadmap methods (\eg, PRM \cite{kavraki1996probabilistic}), and tree methods (\eg, RRT \cite{lavalle1998rapidly} or RRT* \cite{Karaman2011}).
There is also a large literature on algorithms that focus on optimizing the \textit{search} of an existing graph. A* \cite{hart1968} is the {\em de-facto} standard search method with numerous adaptations to provide properties such as planning on dynamic graphs (\eg, D* Lite \cite{Koenig2002}), and heuristically accelerated planning (\eg, ARA* \cite{Likhachev2004}).

The algorithm proposed here lies in the area of graph construction, starting from a single edge from origin to goal and an empty map, and iteratively switching between checking the validity of the solution (mapping), and updating the graph structure to account for newly processed sensor data (planning).

Our approach incrementally grows a sparse graph by taking advantage of two provable properties of the problem, specifically that (1) the constrained optimal trajectory will be made of free space trajectories joined at the boundaries of obstacles and (2) adding obstacles to the problem will never decrease the optimal cost of the motion planning problem.
By not constraining the trajectories to a pre-determined form, we are also able to handle systems with differential constraints, provide a naturally multi-resolution representation of the state space, and create plan graphs that can be efficiently queried, while minimizing the mapping required.
The incremental growth of the plan graph starts from a best case obstacle free solution to the motion planning problem, checks it against available sensor data, and in the event of discovering a new obstacle adds new elements to the plan graph to avoid the obstacle. This process is repeated until an optimal solution is reached.
By mapping along only the current best solution we can perform ``edge optimal graph search'' following the model of Dellin and Srinivasa \cite{Dellin2016} to perform minimal mapping on the way to finding the optimal solution.

This paper is organized as follows. Section \ref{sec:related_work} lays out the related work in the field of graph construction and minimal sensing. Section \ref{sec:problem_definition} defines the optimal path planning problem and the notation used in the paper. Section \ref{sec:algorithm} describes the graph construction algorithm and its foundations. Section \ref{sec:analysis} provides proofs for the completeness and optimality of the algorithm. Finally, Section \ref{sec:experiments} tests the graph construction algorithm in conjunction with the dynamic search algorithm D* Lite \cite{Koenig2002} on several standard systems.

%%% Local Variables:
%%% mode: latex
%%% TeX-master: "main"
%%% End:

%!TEX root = main.tex
\section{Related Work} \label{sec:related_work}

A typical setup for the robotic motion planning problem combines a pre-built map of the environment and a plan graph of nodes (robot states) and edges (robot trajectories).
In common implementations the plan graph uses the map for validity checks of nodes and edges, but does not reference it for the structure of the plan graph~\cite{lavalle2006planning}.
Many methods have been developed, however, that more tightly couple the planning and mapping processes to achieve lower cost trajectories or computationally more efficient planners.
For example, visibility graphs place nodes at the vertices of polygonal obstacles providing exactly optimal solutions for 2D holonomic robots~\cite{Alt1988}.
In sampling based methods, the obstacle map can be used to inform node sampling strategies such as increased placement near obstacle boundaries for navigating cluttered environments~\cite{amato1998obprm}.

Several methods exist which use the result of collision checks executed during search to adapt the structure of the plan graph.
The any angle planning variants Theta* \cite{Daniel2010b}, Lazy Theta* \cite{Nash2010b}, and Incremental Phi* \cite{Likhachev2009b} start with a grid structure for the initial search but add virtual diagonal edges to shorten the path where permissible on the obstacle map.
Of close relation to this work, several planners have been proposed which initially start searching a simple plan graph, and incrementally increase the complexity of the problem based on the result of collision checks.
For example, Wagner and Choset \cite{Wagner2015} propose initializing a $n$-robot planning problem as $n$ 1-robot sub-problems.
When two sub-problems produce collisions between robots, they are combined and re-solved, thereby locally increasing the complexity of the problem but eliminating the collision. This process is repeated until no collisions remain.
Similar concepts have been proposed for other navigation scenarios. Shah \etal~\cite{Shah2016} initialize with a low resolution grid representation of the state space and adaptively increase its resolution during search. Gochev \etal~\cite{gochev2012planning} propose a method that incrementally increases the dimensionality of the state space for complex robots when low dimensional representations cause collisions.

Another field of work has looked at reducing computation by minimizing the number of collision checks that must be carried out during motion planning. These methods, often called ``lazy'' search techniques, typically focus on robotic arms \cite{Cohen2014} but apply to any graph search problem.
Dellin and Srinivasa \cite{Dellin2016} show that several existing algorithms for ``lazy'' search are actually specific instances of a more general algorithm for minimizing the number of collision checks.

Recent work has considered bringing together robot mapping and motion planning.
Pryor \etal~\cite{Pryor2016a} use motion planning to determine which areas of the map around a humanoid robot to resolve from sensor data.
Ghosh and Biswas \cite{Ghosh2017a} show significant reductions in the matching of stereo pairs for a ground robot by directly connecting the checking of disparity matches to the expansion of the plan graph.
These methods provide a strong basis for the benefits of creating joint mapping and planning processes.

%%% Local Variables:
%%% mode: latex
%%% TeX-master: "main"
%%% End:

%!TEX root = main.tex

\section{Problem Definition} \label{sec:problem_definition}

The interval between $a \in \reals$ and $b \in \reals$, where $\reals$ is the set of real numbers, is defined as the set of all real numbers between $a$ and $b$, including $a$, but not including $b$, and denoted by $[a,b)$.

The dynamics governing the robot is represented by an ordinary differential equation of the following form:
\begin{equation} \label{eqn:dynamics}
\dot{\state{}}(t) = f (\state{}(t), u(t) )
\end{equation}
where $\state{}(t) \in \statespace{}$ is the state of the robot and $u(t) \in {\cal U}$ is the input, at time $t$. The sets $\statespace{}$ and ${\cal U}$ are called the state space and the input space, respectively.
A {\em dynamically-feasible trajectory} from a starting state $\state{a} \in \statespace{}$ to an end state $\state{b} \in \statespace{}$ is a mapping $\traj{} : [0,T) \to \statespace{}$, such that $\traj{}(0) = \state{a}$, $\traj{}(T) = \state{b}$ and $\traj{}(t)$ satisfies Equation~\eqref{eqn:dynamics} for all $t \in [0,T)$. Let $\trajset{}(\state{a}, \state{b})$ denote the set of all dynamically-feasible trajectories in $\statespace{}$ from $\state{a}$ to $\state{b}$.

Let ${\cal C} : \traj{} \to \reals \cup \{\infty\} $ be a {\em cost function} that assigns each trajectory with a cost. We also define this cost function for sets of trajectories as the minimal cost of its elements, $\Cost{\trajset{}} = \min_{\traj{} \in \trajset{}} \Cost{\traj{}}$.
The {\em map of obstacles} is an open subset of states, $\map{} \subset \statespace{}$, that the robot can not attain. We assign any trajectory traversing an obstacle infinite cost, {\em i.e.}, $\Cost{\traj{}} = \infty$ for all $\traj{} : [0,T) \to \statespace{}$ such that $\traj{}(t) \in \map{}$ for some $t \in [0,T)$.
We denote the set of dynamically feasible and finite cost trajectories from $\state{a}$ to $\state{b}$, given the map $\map{}$, by $\trajset{}_\map{} (\state{a}, \state{b})$.

We consider the standard {\em optimal motion planning problem}: Given the dynamics of the robot as in Equation~\eqref{eqn:dynamics}, the map of obstacles $\map{}$, the cost function ${\cal C}$, a start state $\state{s} \in \statespace{}$ and a goal state $\state{g} \in \statespace{}$, we are interested in finding a minimum cost trajectory from $\state{s}$ to $\state{g}$, {\em i.e.}, $\argmin_{\traj{} \in \trajset{}(\state{s}, \state{g})} \Cost{\traj{}}$.
If the resulting path has finite cost, it is collision free, if not, there is no feasible trajectory.

We are interested in an algorithm that constructs the map and the plan graph jointly to solve the optimal motion planning problem. To guarantee a finite run time, we assume that the map of obstacles, $\map{}$, consists of the union of a finite number of obstacles.

Often detecting all of the obstacles at the highest possible resolution requires extensive computation time just for the mapping phase. Hence, our goal is to plan trajectories without adding all obstacles to the map, unless it is necessary. Strictly speaking, we generate a sequence of maps, say $\map{0}, \map{1}, \dots, \map{m} \subseteq \map{}$, such that $\map{k} \subseteq \map{k+1}$, and consider the motion planning problem in subsequent maps. From a planning point of view, the map $\map{k}$ is the collection of all obstacles processed up until the $k$th planning phase.

Our goal is the following: {\em (i)} adaptively construct an efficient plan graph; {\em (ii)} ensure completeness and optimality of the returned solution from the plan graph against all available sensor data; {\em (iii)} minimize the amount of sensor data that must be processed into a map to accomplish {\em (ii)}. For this purpose, we develop a joint mapping and planning algorithm that constructs the map and the plan graph in stages by adding the obstacles to the map and the trajectories to the plan graph gradually as needed.

%%% Local Variables:
%%% mode: latex
%%% TeX-master: "main"
%%% End:

%!TEX root = main.tex

\section{Algorithm} \label{sec:algorithm}

In this section, we present a joint map and plan graph construction algorithm. This algorithm constructs a plan graph by adding obstacles in a map as necessary. In a nutshell, the proposed algorithms creates the plan graph using a free-space planner; the path is checked for collision in a lazy way; obstacles that the path intersects are added to the map; and the plan graph is updated locally to account for the new obstacles. In what follows, we formalize the proposed algorithm. % In the next section, we state and prove its guarantees on completeness and optimality.
For this purpose, we first show that any optimal solution is made up of free space optimal trajectories that are joined at obstacle boundaries. Second, motivated by this fact, we present a novel algorithm that jointly constructs a map and a plan graph by adding obstacles into the map as needed.

\paragraph{Trajectory Concatenation}
Before we describe the algorithm, we will define a few properties of adding trajectories and trajectory sets.

Given two trajectories $\traj{1}:[0,T_1) \to \statespace{}$ and $\traj{2} : [0,T_2) \to \statespace{}$, let $(\traj{1} + \traj{2}) : [0,T_1+T_2) \to \statespace{}$ denote the concatenation of $\traj{1}$ and $\traj{2}$, defined as follows:
$$
(\traj{1} + \traj{2})(t) =
\begin{cases}
\traj{1}(t) & \mbox{ for all } t \in [0,T_1);\\
\traj{2}(t-T_1) & \mbox{ for all } t \in [T_1, T_1+T_2).
\end{cases}
$$

Two trajectories $\traj{1}$ and $\traj{2}$ may only be concatenated if $\traj{1}(T_1) = \traj{2}(0)$.

Let $\trajset{}_\map{} (\state{1}, \state{2})$ be the set of all feasible trajectories from $\state{1}$ to $\state{2}$ on the map $\map{}$.
Concatenating two trajectory sets with a shared terminal and origin state, denoted by $\trajset{}_\map{} (\state{1}, \state{2}) + \trajset{}_\map{} (\state{2}, \state{3})$ is formed by concatenating all trajectories in each set.
Note that $\trajset{}_\map{} (\state{1}, \state{2}) + \trajset{}_\map{} (\state{2}, \state{3}) \subseteq \trajset{}_\map{} (\state{1}, \state{3})$ and that $\trajset{}_{\map{1}} (\state{1}, \state{2}) + \trajset{}_{\map{2}} (\state{2}, \state{3}) \subseteq \trajset{}_{\map{1} \cap \map{2}} (\state{1}, \state{3})$.
We define the union of two trajectory sets $\trajset{}_1 (\state{1}, \state{2}) \cup \trajset{}_2 (\state{3}, \state{4})$ as the standard set union, but only to be viable if $\state{1} = \state{3}$ and $\state{2} = \state{4}$.

\paragraph{Candidate solutions}
Let us make two assumptions about the nature of the problem.
\begin{assumption} \label{assumption:triangle}
We assume that the cost function, $\Cost{\cdot}$, satisfies the triangle inequality for all $\traj{} \in \trajset{}$.

\end{assumption}
\begin{assumption} \label{assumption:free_planner}
  We assume the existence of a planner that can generate the (possibly infinite and possibly empty) set of {\em locally minimal} free space trajectories from $\state{1} \in \statespace{}$ to $\state{2} \in \statespace{}$.
  We will denote this set of trajectories $\trajset{}_{\emptyset} (\state{1}, \state{2})$, with the properties that for all $\traj{i}$ in $\trajset{}_{\emptyset} (\state{1}, \state{2})$:
  \begin{equation}
    \begin{aligned}
  \traj{i} \in \trajset{} (\state{1}, \state{2}), \\
  J({\cal C}(\traj{i})) = 0, \, H({\cal C}(\traj{i})) \succeq 0,
  \end{aligned}
  \end{equation}
  where $J$ and $H$ are the Jacobian and Hessian operators.
  
\end{assumption}
Assumptions \ref{assumption:triangle} and \ref{assumption:free_planner} describe many dynamical systems that are commonly referenced in literature, including holonomic robots, integrators of any order, Dubins cars, and Reeds-Shepp cars. In general, Pontryagin's Minimum Principle \cite{pontryagin1964} can be used to determine the necessary conditions for these free space minima.
In many cases the locally minimal trajectory set described in Assumption \ref{assumption:free_planner} contains a single globally minimal trajectory, however, in cases such as the Dubins car there may be multiple (the six Dubins paths)~\cite{dubins1957curves}.

Given these assumptions we can make a statement about the structure of any optimal solution to the motion planning problem.
\begin{theorem}\label{theorem:complete_graph}
Suppose Assumption~\ref{assumption:triangle} and Assumption~\ref{assumption:free_planner} hold.
If it exists, any optimal path from \instate{s} to \instate{g} is made up of a finite number of locally minimal trajectories, with joints at the obstacle boundaries.

\end{theorem}
\begin{proof}
A similar result is stated in Theorem 25 of Pontryagin's Mathematical Theory of Optimal Processes \cite{pontryagin1964}, and is restated here with translations to our terminology in brackets for clarity:
``Let the optimal trajectory [of Equation~\eqref{eqn:dynamics}] lie wholly in the closed domain [$\statespace{} \setminus \map{}$] and contain a finite number of points of abutment [entrances to the boundary], and let every piece of it that lies on the boundary of G be regular. Then every piece of trajectory in the open kernel of [$\statespace{} \setminus \map{}$] (with the possible exception of its ends) satisfies the [minimum] principle; every piece lying on the boundary of [$\statespace{} \setminus \map{}$] satisfies Theorem 22; and the jump condition (Theorem 24) is satisfied at every point of abutment.''
The theorem we present is a direct relaxation of Theorem 25 without the jump condition constraining the concatenation of trajectories.
\end{proof}

Based on Theorem~\ref{theorem:complete_graph}, we can describe a graph denoted $G_\mathrm{complete}$ that is guaranteed to contain any optimal solution on a map $\map{}$.
Let ${\cal B}(\map{})$ denote the boundary of the open set $\map{}$.
Consider the (possibly infinite) graph constructed in the following manner. The set of nodes is defined as $\state{s} \cup \state{g} \cup {\cal B}(\map{})$, where $\state{s}$ is the start state, $\state{g}$ is the goal state, and ${\cal B}(\map{})$ is the set of all states that lie at the boundary of the obstacles $\map{}$. The set of edges are all edges between any pairs of nodes, say $(\state{1}, \state{2})$, such that there exists a finite-cost ({\em i.e.}, obstacle-free) trajectory from $\state{1}$ to $\state{2}$, with the edge cost matching the minimal cost valid element of $\trajset{}_{\emptyset} (\state{1}, \state{2})$. $G_\mathrm{complete}$ therefore contains all possible combinations of free space trajectories joined at the map boundaries, guaranteeing from Theorem~\ref{theorem:complete_graph} that it contains any optimal solution that exists.
From $G_\mathrm{complete}$ we can also define a set $\trajset{}_\map{}^*(\state{s}, \state{g})$ that contains all possible traversals from $\state{s}$ to $\state{g}$ in $G_\mathrm{complete}$.
$\trajset{}_\map{}^*(\state{s}, \state{g})$ will either be empty, or its minimum cost element is an optimal solution to the motion planning problem.
% %
% Let us formalize this result.
% %

While $G_\mathrm{complete}$ is guaranteed to contain an optimal path, depending on the properties of $\Bound{}$ it is likely to be uncountable infinite.
To create a computationally tractable problem we discretize $\Bound{}$ with some fixed interval $\delta$ to create a finite approximation of $G_\mathrm{complete}$, $\tilde{G}_\mathrm{complete}$.
It is from this graph that we wish to find a solution.
It should be noted that the discretization only occurs at the boundaries of obstacles, while leaving the majority of state space continuous.

\paragraph{Sparse graph construction}

While $\tilde{G}_\mathrm{complete}$ is finite, it is still impractically large to search or construct in most reasonable instances, instead we propose an algorithm for the creation of a computationally tractable sub-graph, $G_\mathrm{sparse}$, that maintains the guarantee of containing an optimal motion plan if one exists.
This sparse sub-graph is generated by adding obstacles into the map only as needed. The algorithm and principles behind it are described in this section, while proof of completeness and optimality are deferred to Section \ref{sec:analysis}.

The proposed algorithm is presented in Algorithm \ref{algorithm:sparse_graph}.
The goal of the algorithm is to break the larger path planning problem of a minimal trajectory between \instate{s} and \instate{g} into many sub-problems, written \inProblem{a}{b}, of finding the minimal trajectory between two states \instate{a} and \instate{b}.
The algorithm forms a solution by concatenating the solutions to one or more sub-problems with sequential start and end states.
Each sub-problem maintains a \textit{map} of the subset of obstacles that it is aware of, and a list of \textit{parents} which are other sub-problems that may use this sub-problem as part of their solution.
The algorithm accesses mapping, in the form of collision checks, to drive the creation of sub-problems and the addition of new obstacles to sub-maps. For this work we abstract away the nature of the mapping, and simply assume a generic mapping system that can perceive whether an area is free or blocked from sensor data.

At any given iteration, each sub-problem maintains a lower-bound solution given its current sub-map, which is guaranteed to be less than or equal to the solution to the sub-problem on the full map.
At each iteration the sub-maps are grown based on the best solution from the last iteration, increasing the lower-bound solution until the true optimal solution is found.

We remark at this point that Algorithm \ref{algorithm:sparse_graph} is purely a graph construction algorithm, and therefore any optimal graph search algorithm is appropriate for solving the graph at line \ref{sparse_graph:loop_start}, however, due to the continually changing nature of the graph it is best suited to use a dynamic graph solver.
In this work, we used an implementation of D* Lite \cite{Koenig2002}.

\begin{algorithm}[h]
  \caption{Sparse Graph}
  \label{algorithm:sparse_graph}
  %%% Problem setup
  \DontPrintSemicolon
  \SetKwProg{myproc}{Procedure}{}{}
  \SetKwProg{myalgo}{Algorithm}{}{}
  \SetKwFunction{addNodes}{addNodes}
  \SetKwFunction{addEdges}{addEdges}
  \SetKwFunction{Solve}{Solve}
  \SetKwFunction{blocked}{blocked}
  \SetKwFunction{addObstacle}{addObstacle}
  \SetKwFunction{addProblem}{addProblem}
  \SetKwFunction{solve}{SparseShortestPath}

  %%% Solve problem
  % \KwResult{ $\trajset{}_\map{} (\state{s}, \state{g})$ }
  \KwResult{ $\traj{}^* = \argmin_{\traj{} \in \trajset{}(\state{s}, \state{g})} \Cost{\traj{}}$ }
  \myalgo{\solve{\instate{s}, \instate{g}}}{
  \tcc{Add the start and goal}
  \addNodes{$G_\mathrm{sparse}$, $\{ \state{s}, \state{g} \}$}\;
  \addProblem{\instate{s}, \instate{g}}\;
  % G.addProblem(\$\MinTraj{}{s}{g}$)\;
  ${\map{added}} = \emptyset$\;
  \While{true} {
    \tcc{Solve with any optimal graph search algorithm}
    $\traj{0} + \dots + \traj{n}$ = \Solve($G_\mathrm{sparse}$)\; \label{sparse_graph:loop_start}
    \tcc{Check the solution}
    \For{$ \traj{i} \in \traj{}$ } {
      \For{ $\obs{j} \in \obs{}$ }{
        \If { \blocked{$\traj{i}, \obs{j}$} }{
          $\Cost{\traj{i}} = \infty$\;
          \If { $\obs{j} \notin \obs{added}$ }{
            \addNodes{$G_\mathrm{sparse}$, $\DeltaBound{j}$}\;
          }
          $\state{a} = \traj{i}(0)$, $\state{b} = \traj{i}(T_i)$\;
          \addObstacle{\inProblem{a}{b}, $\obs{j}$}\;
          Go to line \ref{sparse_graph:loop_start}\;
        }
      }
    }
    \tcc{Solution is unblocked, then it is optimal}
    % \Return $\traj{} \in \trajset{}_\map{} (\state{s}, \state{g})$\;
    \Return $\traj{0} + \dots + \traj{n}$
  }
}

%%% Add obstacle
  % \setcounter{AlgoLine}{0}
  \myproc{\addObstacle{\inProblem{a}{b}, $\obs{j}$}}{
    % \Procedure{addObstacle}{G, \inProblem{a}{b}, $\obs{j}$}{
    \If{ $\obs{j} \notin \Problem{a}{b}$.map }{
      \inProblem{a}{b}.map += $\obs{j}$\;
      \For{ $\state{k} \in \DeltaBound{j}$ }{
        \addProblem{\instate{a}, \instate{k}}\;
        \inProblem{a}{k}.parents += \inProblem{a}{b}\;
        \inProblem{a}{b}.children += \inProblem{a}{k}\;
        \addProblem{\instate{k}, \instate{b}}\;
        \inProblem{k}{b}.parents += \inProblem{a}{b}\;
        \inProblem{a}{b}.children += \inProblem{k}{b}\;

        \tcc{Recursively add to parents of the path}
        \For {\inProblem{c}{d} $\in$ \inProblem{a}{b} $\cup$ \inProblem{a}{b}.parents } {
          \For { $\obs{l} \in \obs{j} \cup$ \inProblem{a}{k}.map $\cup$ \inProblem{k}{b}.map }{
            \addObstacle{\inProblem{a}{b}, $\obs{l}$}\;
          }
        }
      }
    }
  }

  % \setcounter{AlgoLine}{0}
  %%% Add edge
  \myproc{\addProblem{\instate{a}, \instate{b}}}{
    \If { \inProblem{a}{b} does not exist}{
      \inProblem{a}{b}.parents = $\emptyset$\;
      \inProblem{a}{b}.children = $\emptyset$\;
      \inProblem{a}{b}.map = $\emptyset$ \;
      \addEdges{$G_\mathrm{sparse}$, $\trajset{}_{\emptyset} (\state{a}, \state{b})$}\;
    }
  }

\end{algorithm}

%%% Local Variables:
%%% mode: latex
%%% TeX-master: "main"
%%% End:

%!TEX root = main.tex

\section{Analysis} \label{sec:analysis}

This section will prove the completeness (Theorem \ref{theorem:sparse_completeness}) and optimality (Theorem \ref{theorem:sparse_optimality}) of Algorithm \ref{algorithm:sparse_graph} through several
lemmas on the nature of optimal solutions.

\begin{theorem}[Resolution Completeness of Algorithm \ref{algorithm:sparse_graph}]
  \label{theorem:sparse_completeness}
  Algorithm \ref{algorithm:sparse_graph} produces a solution if and only if one exists, given that the boundary discretization $\delta$ is small enough.
\end{theorem}
\begin{theorem}[Optimality of Algorithm \ref{algorithm:sparse_graph}]
  \label{theorem:sparse_optimality}
  If the graph search algorithm used for Solve is optimal, then the solution returned by Algorithm \ref{algorithm:sparse_graph} is the optimal solution to the global optimization problem.
\end{theorem}

\subsection{Lower Bounding Sets}

To prove these theorems, we shall first prove several lemmas on the nature of lower bounding solutions, then in Section \ref{sec:sparse_equivalence} we show that Algorithm \ref{algorithm:sparse_graph} fits these characteristics, \ie, $G_{sparse}$ always has a solution less than or equal to $G_{complete}$ (optimality), and $G_{sparse}$ will always grow to contain a solution in finite time if a valid solution exists (completeness).

\begin{lemma}[More Obstacles]
  The optimal trajectory on a sub-map is always less than or equal to the optimal trajectory on the full map.
  \begin{equation}
    \begin{aligned}
  \text{If}&,& \, \map{1} &\subseteq \map{2}, \\
  \text{then}&,& \Cost{\trajset{}_{\map{1}}^* (\state{a}, \state{b})} &\leq \Cost{\trajset{}_{\map{2}}^* (\state{a}, \state{b})}.
  \end{aligned}
  \end{equation}
  \label{lemma:more_obstacles}
\end{lemma}
\begin{proof}
  Trivial, adding regions of infinite cost to the state space can only make paths longer.
\end{proof}

Define a set $\trajset{}_\map{}^= (\state{s}, \state{g})$ to be the union of (i) locally minimal trajectories with no obstacles, \ie~$\trajset{}_{\emptyset}(\state{s}, \state{g})$, and (ii) concatenations of two optimal trajectories sets through $\obs{}$ which are joined at $\state{j} \in \DeltaBound{}$.

\begin{equation}
  \begin{aligned}
    \trajset{}_\map{}^= (\state{s}, \state{g}) = \\
    \trajset{}_{\emptyset} (\state{s}, \state{g}) \cup
    \left( \trajset{}_\map{} (\state{s}, \state{i}) + \trajset{}_\map{} (\state{i}, \state{g}) \right) \\
    \; \forall \state{i} \in \DeltaBound{}
  \end{aligned}
\end{equation}

\begin{lemma}[Jointed Trajectory]
  The set $\trajset{}_\map{}^= (\state{s}, \state{g})$ contains an optimal trajectory on $\map{}$, i.e.
  $\Cost{\trajset{}_\map{}^=(\state{s}, \state{g})} = \Cost{\trajset{}_{\map{}}^* (\state{s}, \state{g})}$
    % \end{aligned*}
  \label{lemma:jointed_path}
\end{lemma}

\begin{proof}
  This lemma follows directly from Theorem \ref{theorem:complete_graph}, which states that the optimal path is either a free space solution to the problem, or contains a joint at {\em at least} one point on the boundary of $\map{}$.
\end{proof}

%%% Jointed Sub maps

As discussed in Section \ref{sec:algorithm} our goal is not to have a set of optimal trajectories, but instead to have a set that {\em lower-bounds} the optimal trajectory.
To that end we define another set $\trajset{}_\map{}^< (\state{s}, \state{g})$ by:
\begin{equation}
  \begin{aligned}
    \trajset{}_\map{}^< (\state{s}, \state{g}) = \\
    \trajset{}_{\emptyset} (\state{s}, \state{g}) \cup
    \left( \trajset{}_{{\map{si}}} (\state{s}, \state{i}) + \trajset{}_{{\map{ig}}} (\state{i}, \state{g}) \right) \\
    \; \forall \state{i} \in \DeltaBound{}, \, \exists {\map{si}} \subseteq \map{}, \, \exists {\map{ig}} \subseteq \map{}
  \end{aligned}
\end{equation}

\begin{lemma}[Sub-map Trajectory]
  The cost of the minimal element of $\trajset{}_\map{}^< (\state{s}, \state{g})$ is less than or equal to the cost of the minimal element of $\trajset{}_\map{}^= (\state{s}, \state{g})$.
  \label{lemma:incomplete_path}
\end{lemma}
\begin{proof}
  The case of the minimal element being in $\trajset{}_{\emptyset} (\state{s}, \state{g})$ is trivial. Otherwise,
  there is a one-to-one mapping between the elements of $\trajset{}_\map{}^< (\state{s}, \state{g})$ and $\trajset{}_\map{}^= (\state{s}, \state{g})$ based on the joint state $\state{i}$.
  Let $\traj{}^=$  be the minimal element of
  $\trajset{}_\map{}^= (\state{s}, \state{g})$, where $\traj{}^= \in \trajset{}_\map{} (\state{s}, \state{i}) + \trajset{}_\map{} (\state{i}, \state{g})$.

  From Lemma \ref{lemma:more_obstacles} we know that if $\traj{}$ is the minimal element of $\trajset{}_{{\map{si}}} (\state{s}, \state{i}) + \trajset{}_{{\map{ig}}} (\state{i}, \state{g})$ and ${\map{si}} \subseteq \map{}, \, {\map{ig}} \subseteq \map{}$, then $\Cost{\traj{}} \leq \Cost{\traj{}^=}$. Therefore there is some element in $\trajset{}_\map{}^< (\state{s}, \state{g})$ that is less than or equal to the minimal element in $\trajset{}_\map{}^= (\state{s}, \state{g})$.

\end{proof}

Finally, we will define one more lower bounding set to be:
\begin{equation}
  \begin{aligned}
    \trajset{}_\map{}^{<<} (\state{s}, \state{g}) = \\
     \trajset{}_{\emptyset} (\state{s}, \state{g})  \cup
    \left( \trajset{}_{{\map{si}}}^{<} (\state{s}, \state{i}) + \trajset{}_{{\map{ig}}}^{<} (\state{i}, \state{g}) \right) \\
    \; \forall \state{i} \in \DeltaBound{}, \, \exists {\map{si}} \subseteq \map{}, \, \exists {\map{ig}} \subseteq \map{}
  \end{aligned}
\end{equation}

\begin{lemma}[Set of Incomplete Sets]

  The minimum cost element of $\trajset{}_\map{}^{<<} (\state{s}, \state{g})$ has a cost less than or equal to the minimum cost element of $\trajset{}_\map{}^{<} (\state{s}, \state{g})$.
  \label{lemma:set_of_sets}
\end{lemma}
\begin{proof}
  The proof can be found in the same manner as Lemma \ref{lemma:incomplete_path} by making one-to-one comparisons of elements, and therefore is omitted for brevity.

\end{proof}

While we could continue describing sets in this manner, a new set $\trajset{}_\map{}^{<<<} (\state{s}, \state{g})$ has the same characteristics as the set $\trajset{}_\map{}^{<<} (\state{s}, \state{g})$.

\subsection{Equivalence to Sparse Graph Algorithm} \label{sec:sparse_equivalence}

In Lemmas \ref{lemma:more_obstacles}-\ref{lemma:set_of_sets} we described a type of trajectory set, $\trajset{}_\map{}^{<<} (\state{s}, \state{g})$, whose minimal element always provides a lower bound to the true optimal solution.
We will show that the graph built by Algorithm~\ref{algorithm:sparse_graph}, $G_\mathrm{sparse}$, has the properties of $\trajset{}_\map{}^{<<}$, and therefore lower bounds the optimal solution.
To do so we will equate the sub-problems \inProblem{a}{b} to trajectory sets $\trajset{} (\state{a}, \state{b})$. We will call a sub-problem {\em lower-bounding} if it has the form of $\trajset{}_\map{}^{<<} (\state{a}, \state{b})$.
This means that there are edges in $G_{sparse}$ corresponding to $\trajset{}_{\emptyset}(\state{a}, \state{b})$, and for every sub-problem \inProblem{}{} in the children of \inProblem{a}{b}, \inProblem{}{}.map $\subseteq$ \inProblem{a}{b}.map, and \inProblem{}{} is also {\em lower-bounding}.

\begin{theorem}
  \label{theorem:sparse_lower_bound}
  At every iteration of Algorithm \ref{algorithm:sparse_graph}, \inProblem{s}{g} is lower-bounding, therefore $G_{sparse}$ is equivalent to $\trajset{}_\map{}^{<<} (\state{s}, \state{g})$.
\end{theorem}
\begin{proof}
  We shall prove this by induction.

  At iteration 0, \inProblem{s}{g}.map = $\emptyset$, therefore $\trajset{}_\map{}^{<<} (\state{s}, \state{g}) = \trajset{}_{\emptyset} (\state{s}, \state{g})$, which is added to the graph at Line 3.

  At iteration $k$, assume that \inProblem{s}{g} is lower-bounding. Because of the definition of lower-bounding, this means that every sub-problem is also lower-bounding.

  At iteration $k+1$, let the blocked problem be labeled \inProblem{a}{b}, which is blocked by $\obs{j}$. If $\obs{j} \in$ \inProblem{a}{b}.map there is no change to the structure of the graph, therefore \inProblem{s}{b} remains lower-bounding.
  Otherwise, $\obs{j}$ will be incorporated into \inProblem{a}{b}.
  There will be some lower-bounding path to every state on the boundary of $\obs{j}$ from the function $\mathsf{addProblem}$, each of which will have \inProblem{a}{b} as a parent.
  From Lines 27-29 the sub-maps of child problems will always be subsets of the maps of parent problems.
  Since $\mathsf{addObstacle}$ is recursive, lower-bounding will be carried through the chain of parents to \inProblem{s}{g}, ensuring that each remains lower-bounding.

\end{proof}

Given that $G_{sparse}$ is equivalent to $\trajset{}_\map{}^{<<} (\state{s}, \state{g})$ we will prove the optimality and completeness of the algorithm.

\begin{proof}[Optimality of Algorithm \ref{algorithm:sparse_graph}]
  From Theorem \ref{theorem:sparse_lower_bound} we have that $G_{sparse} \sim \trajset{}_\map{}^{<<} (\state{s}, \state{g})$, and from Lemmas \ref{lemma:more_obstacles}-\ref{lemma:set_of_sets} we know that the minimum element of $\trajset{}_\map{}^{<<} (\state{s}, \state{g})$ is less than or equal to the optimal solution $\traj{}^* = \argmin_{\traj{} \in \trajset{}(\state{s}, \state{g})} \Cost{\traj{}}$, therefore if Algorithm \ref{algorithm:sparse_graph} returns a solution, it must have a cost less than or equal to the optimal solution.

  By definition there is no valid trajectory with cost less than the optimal solution, therefore if Algorithm \ref{algorithm:sparse_graph} returns a solution, it must be an optimal solution.
\end{proof}

\begin{proof}[Resolution Completeness of Algorithm \ref{algorithm:sparse_graph}]
  At every iteration, either the algorithm returns a solution, or one edge of the graph will be found to pass through an obstacle and be set to infinity.
  There are finitely many possible edges in the graph (the total number of edges in $G_{complete}$), therefore Algorithm \ref{algorithm:sparse_graph} is guaranteed to return a (possibly infinite) solution after $\text{Size}(G_{complete})$ iterations.
  From Theorem~\ref{theorem:sparse_optimality}, Algorithm~\ref{algorithm:sparse_graph} will only return an invalid (infinite cost) solution if no valid solution exists.
  Therefore, Algorithm~\ref{algorithm:sparse_graph} will always return a valid solution if one exists.

\end{proof}

%%% Local Variables:
%%% mode: latex
%%% TeX-master: "main"
%%% End:

%!TEX root = ms.tex

\section{Computational Experiments} \label{sec:experiments}

To test the validity and effectiveness of our algorithm we compared it against the results of a standard grid based graph construction method, searched using lazy edge checking and D$^*$ Lite \cite{Koenig2002}.
We compare the two algorithms for three simple dynamical systems, namely, holonomic robots in 2D and 3D, and a Dubins car.

\subsection{Simulation Setup}
\begin{figure}[]
  \vspace{0.1cm}
  \centering
  \includegraphics[width=0.9\linewidth,trim={8cm 2cm 22cm 12cm},clip]{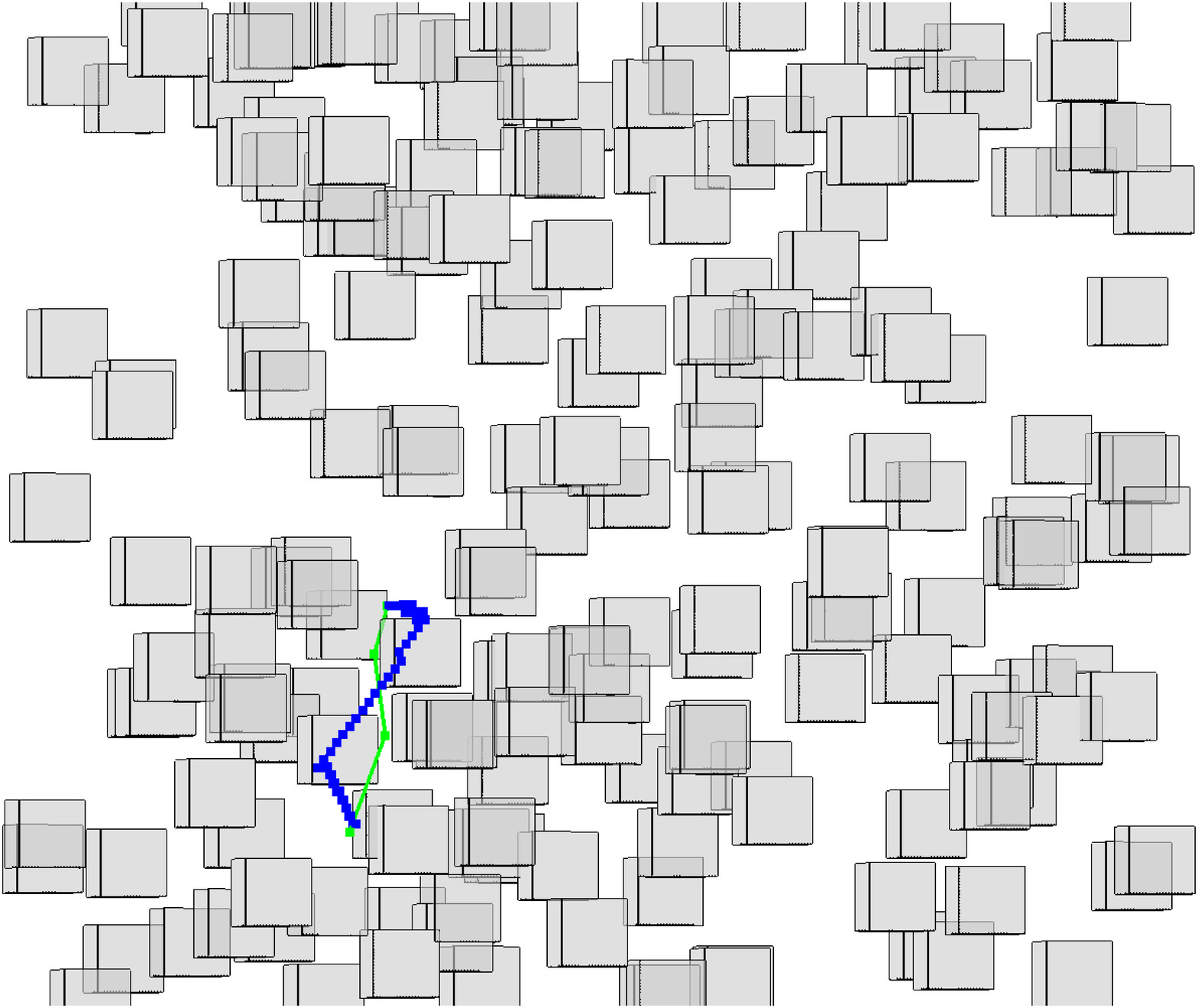}
  \caption{View of a 3D simulation with 200 randomly spaced cubes (gray) of length 2. The trajectory generated from a sparse plan graph is shown in green, and from a grid plan graph in blue. Both graphs use a spatial discretization of 0.25, and the grid plan graph uses a connectivity of 1 (26 connected).}
  \label{fig:geo3d_viewer}
\end{figure}
\begin{figure}
  \vspace{0.1cm}
    \centering
    \includegraphics[width=0.9\linewidth,trim={8cm -1cm 8cm 2cm},clip]{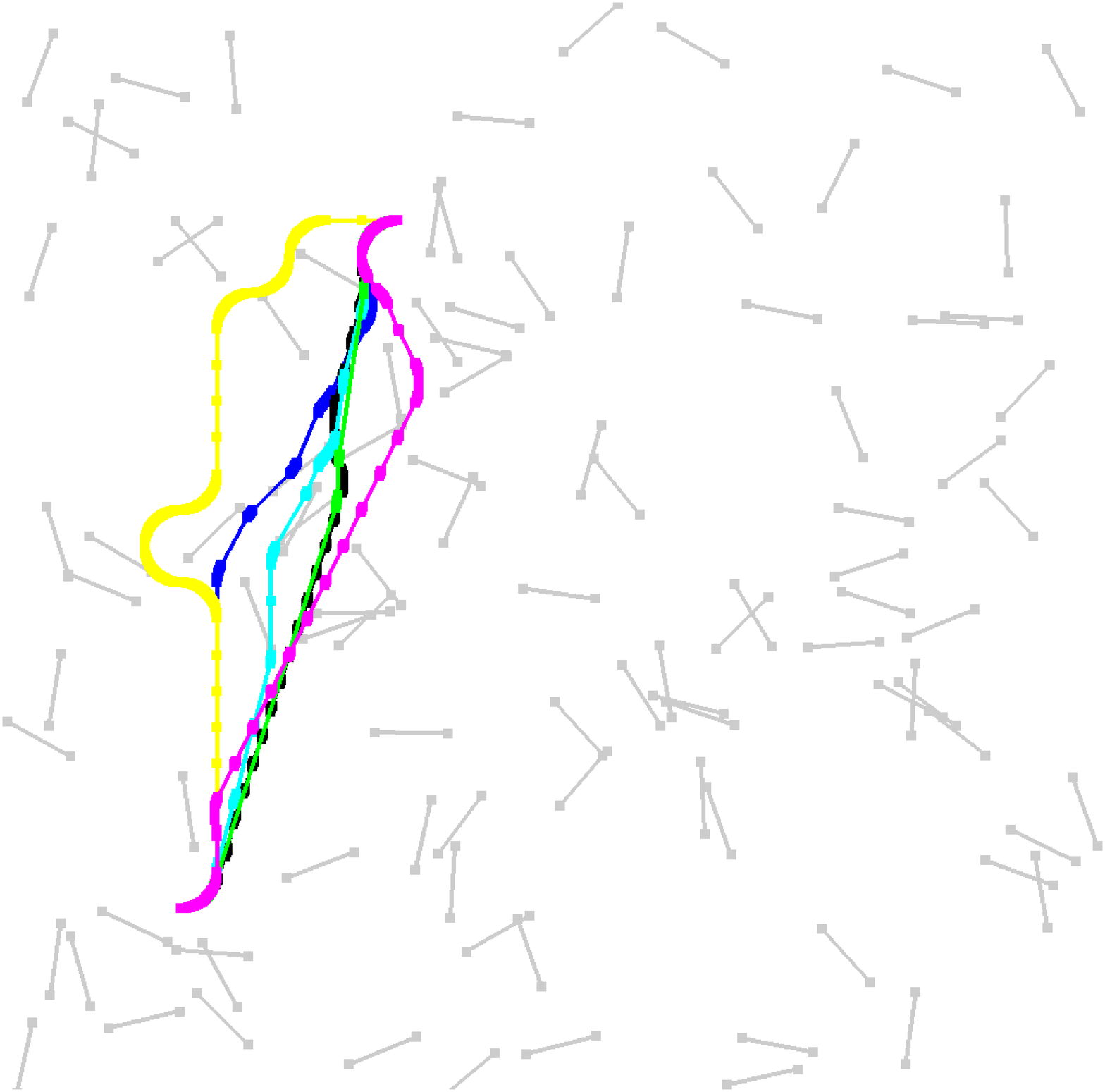}
% \end{figure}
% \begin{table}
%   \centering
  \resizebox{\columnwidth}{!}{
    \begin{tabular}{ | c | c r r r | r r |}
    \hline
  % \label{table:dubins_example}

  Color & Planner & Grid & Angular & Conn. & Cost & Time \\
        &         & Discr. & Discr.    &  & & (ms) \\
  \hline
  Green & Sparse & -- & $\pi / 8$ & -- & \textbf{20.71} & \textbf{16.9} \\
  Black & Grid & 0.25 & $\pi / 8$ & 4 & 20.86 & 751.3 \\
  Cyan & Grid & 0.5 & $\pi / 8$ & 4 & 20.88 & 486.7 \\
  Blue & Grid & 1.0 & $\pi / 4$ & 2 & 21.43 & 49.1 \\
  Purple & Grid & 0.5 & $\pi / 8$ & 2 & 21.86 & 110.9 \\
  Yellow & Grid & 1.0 & $\pi / 2$ & 0 & 25.56 & 25.6 \\
\hline
  \end{tabular}
}
  \caption{\label{fig:dubins_viewer} Generated trajectories of a Dubins car with turning radius 1, traveling through 100 obstacles of length 2, using 6 different plan graphs. The different graph types, their display color, and the computed trajectory cost and computation time are shown in the table.
  }

\end{figure}

Computational experiments were performed in simulations involving randomly generated obstacle fields in $\Real{2}$ (Holonomic 2D and Dubins Car) and $\Real{3}$.
For the 2D holonomic robot, obstacles are impassible line segments of a fixed length and random orientation distributed randomly in $[0, 30]\text{x}[0, 30] \in \Real{2}$.
For the 3D holonomic robot obstacles are impassible cubes of fixed side length randomly distributed in $[0, 30]\text{x}[0, 30]\text{x}[0, 30] \in \Real{3}$.
The starting location is $\state{s} = (5, 5, 5)$ and the goal location was set 20 units away at a random direction in the positive quadrant, with the third dimension ignored for the 2D cases.
The goal location is rounded to the nearest integer to allow for easy integration with grid based methods.
For Dubin's cars the initial and final orientations were randomly generated increments of $\pi/2$, again to allow for easy integration with grid based methods.
Examples of the setup can be seen in Figures \ref{fig:geo3d_viewer} and \ref{fig:dubins_viewer}.
% \begin{figure*}[H]
\begin{figure*}
  \vspace{0.2cm}
  \centering
  \begin{subfigure}{0.32\textwidth}
  \includegraphics[width=1.0\linewidth, trim={1.5cm 2.0cm 2.3cm 0.5cm},clip]{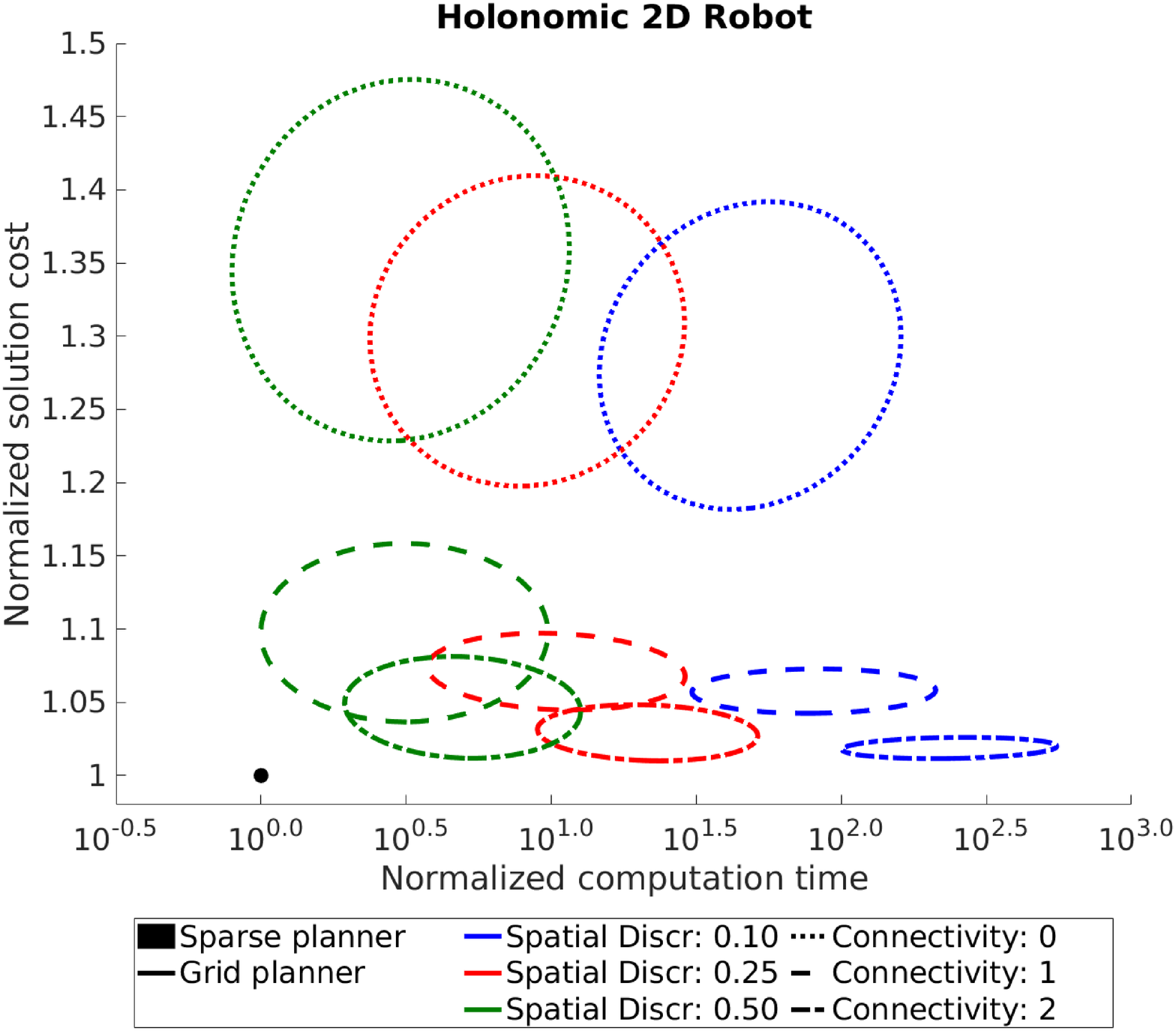}
  \caption{Trajectory cost vs. computation time for a holonomic 2D robot moving through 100 obstacles of length 2. Cost and trajectory length for each map is normalized by the values of planning with the sparse plan graph.}
    \label{fig:geo2d}
  \end{subfigure}
  \hfill
% \begin{figure}[h]
    % \vspace{.3cm}
    % \centering
    \begin{subfigure}{0.32\textwidth}

  \includegraphics[width=1.0\linewidth, trim={1.5cm 2.0cm 2.3cm 0.5cm},clip]{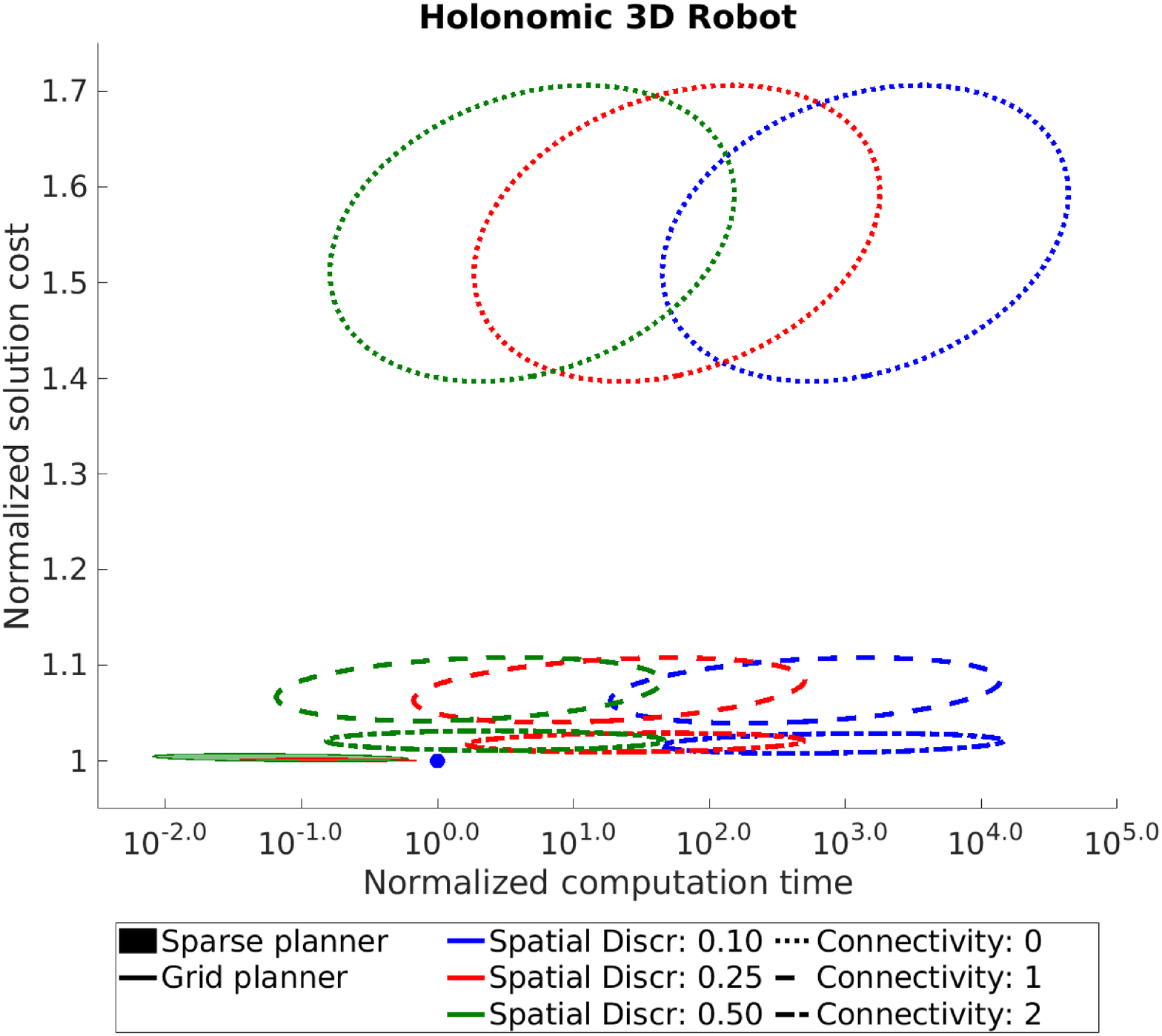}
  \caption{Trajectory cost vs. computation time for a holonomic 3D robot moving through 200 obstacles of side length 2. Cost and trajectory length for each map is normalized by the values of planning with the sparse plan graph with spatial discretization of 0.1.}
  \label{fig:geo3d}
  \vspace{-0.35cm}
\end{subfigure}
\hfill
% \begin{figure}[h]
    % \vspace{.3cm}
    % \centering
\begin{subfigure}{0.32\textwidth}
  \includegraphics[width=1.0\linewidth, trim={1.5cm 2.0cm 2.3cm 0.3cm},clip]{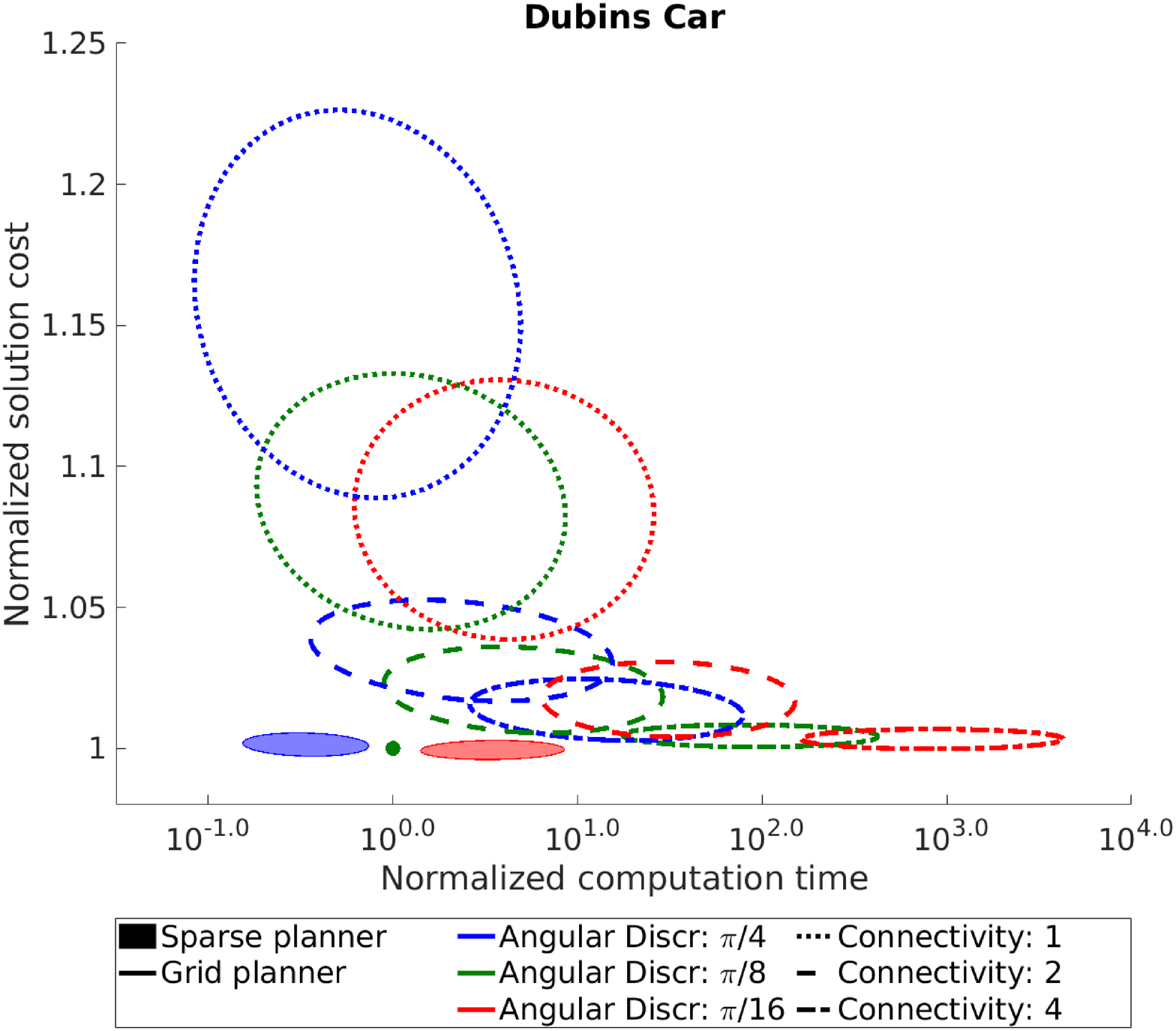}
  \caption{Trajectory cost vs. computation time for a Dubins car moving through 100 obstacles of length 2. Cost and trajectory length for each map is normalized by the values of planning with the sparse plan graph with angular discretization of $\pi/8$.}
  \label{fig:dubins}
    \vspace{-0.33cm}
  \end{subfigure}
  \vspace{0.3cm}
\caption{\label{fig:all_results}Normalized results for trajectory cost vs. computation time for three robots on 200 randomly generated maps. Each type of plan graph was run on the same 200 maps for each robot, and the results were normalized to show the relative speed and solution quality of the different graphs. The figures display the 95\% boundary of the values for each graph type, with shape fill showing the type of plan graph (sparse or grid), color showing the discretization level, and line style showing the connectivity of the grid graphs.}
  \end{figure*}

\subsection{Grid Construction}

The gridded graph construction places nodes at regular intervals through the state space with a fixed spatial discretization and, for the Dubins car, angular discretization.
The edges between nodes are placed based on a connectivity parameter, which determines to what level of adjacent nodes a node is connected to.
A connectivity level of ``0'' connects to adjacent non-diagonal nodes (4 connected in 2D) and a connectivity of $n$ for $n \geq 1$ connects to all nodes within $n$ of that node in any dimension.
For holonomic robots any path that was a scalar duplicate of another one was trimmed to limit complexity.
The Dubins car is connected to all discretizations of angular space within the spatial connectivity, but with a heuristic pruning of high cost (near full turn) trajectory primitives. This was found to significantly speed up computation without hurting the quality of the paths.
  \begin{table*}[h]
  \vspace{-0.0cm}
  \centering
  \begin{tabular}{| c  r  r  r | r r r r r |}
    \hline
    Planner & Grid Discretization & Angular Discretization & Connectivity & Path Cost & Plan Time (ms) & Nodes & Edges & Area Sensed \\
    \hline
    Sparse & -- & $\pi / 16$ & -- & 22.315 & 1645 & 282 & 4838 & 436 \\
    Sparse & -- & $\pi / 8$ & -- & 22.328 & 140 & 140 &              1369 &      418 \\
    Sparse & -- & $\pi / 4$ & -- & 22.357 & 20 & 70 &              383 &       382 \\
    Grid & 0.25 & $\pi / 16$ & 4 & 22.400 & 35871 & 49318 &          1101100 &     528 \\
    Grid & 0.25 & $\pi / 8$ & 4 & 22.424 & 2541 & 25276   &       278740   &   513 \\
    Grid & 0.25 & $\pi / 4$ & 4 & 22.635 & 258 & 13704 &               77655 &      522 \\
    Grid & 0.50 & $\pi / 16$ & 2 & 22.710 & 715 & 13069 &         113290  &    544 \\
    Grid & 0.50 & $\pi / 8$ & 2 & 22.783 & 78 & 6677 &               27560 &      546 \\
    Grid & 0.50 & $\pi / 4$ & 2 & 23.090 & 29 & 3954 &               11733 &      449 \\
    Grid & 1.00 & $\pi / 16$ & 1 & 24.214 & 55 & 4266 &               19424 &      417 \\
    Grid & 1.00 & $\pi / 8$ & 1 & 25.280 & 15 & 2287 &                5897 &      403 \\
    Grid & 1.00 & $\pi / 4$ & 1 & 25.837 & 7 & 887 &              1780 &      282 \\

    \hline
  \end{tabular}
  \caption{Summary of mean values for planning a trajectory for a Dubins Car with turning radius of 1 amoung 100 obstacles of length 2 (see Figure \ref{fig:dubins_viewer}). Standard deviations are omitted as individual measurements are map dependent and therefore the values do not follow a normal distribution, see Figure~\ref{fig:dubins} for relative distributions.}
    \label{table:dubins_data}
\end{table*}
\subsection{Search}

Both the graph created by Algorithm~\ref{algorithm:sparse_graph} and the grid graph use the same implementation of D$^*$ Lite \cite{Koenig2002} as the search algorithm.
As we are only performing a single planning step we only use the lifelong planning element of the algorithm, with edge updates coming from the ``lazy'' collision checking.
Edges for collision checking along the current ``best path'' are selected starting from the robot's current pose and continuing forward towards the goal.
This matches Lazy Weighted A* \cite{Cohen2014} and the forward edge selector described by Dellin and Srinivasa~\cite{Dellin2016}.
As suggested by Dellin and Srinivasa~\cite{Dellin2016} other edge selectors are viable and can have different performance characteristics, however, that is not the focus of this work.
Collision checks in 2D were performed using a simple line intersection check, while collision checks in 3D were performed using ray casting in Octomap \cite{Hornung2013}. Each collision checker acts as a simulated sensor, moving along the trajectory mapping free or occupied space.

\subsection{Results}
Comparison data for the three dynamical systems is shown in Figure~\ref{fig:all_results}, comparing the length of the path generated with the time to generate the path.
Since path length and computation time are highly map dependent, the values are normalized by the results from a single method and the same scenario is re-run with multiple planners/settings.
As can be seen in the results, while there is a clear trade off between path quality and solution time in the grid based methods, the sparse plan graph method is significantly less sensitive to parameter choice in path length, and provides paths that are both lower in cost and faster to compute.

A summary of the data for the Dubins car experiments is shown in Table \ref{table:dubins_data}. As expected, the sparse graph construction algorithm created graphs that generated shorter paths, faster, with significantly less nodes and edges. The total area sensed was measured by discretizing the position space into voxels of size 0.2 and marking them as sensed if the mapping process moved through them.

%%% Local Variables:
%%% mode: latex
%%% TeX-master: "ms"
%%% End:

\section{Conclusion}

In this paper we describe a new plan graph construction algorithm which explicitly connects the mapping and planning elements of the navigation problem.
We describe this algorithm in detail, and prove that it generates graphs that are both optimal and complete against the full sensor data, despite only having directly mapped a subset of it.

The graph uses discretization only along the boundaries of the state space, while remaining continuous on the inner open set, providing lower cost paths than a discritization of the full state space. In addition, by using mapping to drive the construction of the graph, nodes and edges are added ``as needed'' providing a sparse representation of the underlying motion planning problem for fast trajectory computation.

%%% Local Variables:
%%% mode: latex
%%% TeX-master: "main"
%%% End:

\balance
\bibliographystyle{IEEEtran}
\bibliography{refs}

\end{document}